\documentclass[english]{article}
\pdfoutput=1
\usepackage[T1]{fontenc}
\usepackage[latin9]{inputenc}
\usepackage{color}
\usepackage{float}
\usepackage{amsmath}
\usepackage{amsthm}
\usepackage{amssymb}
\usepackage{graphicx}
\usepackage{setspace}

\makeatletter

\floatstyle{ruled}
\newfloat{algorithm}{tbp}{loa}
\providecommand{\algorithmname}{Algorithm}
\floatname{algorithm}{\protect\algorithmname}

\theoremstyle{plain}
\newtheorem{thm}{\protect\theoremname}
  \theoremstyle{plain}
  \newtheorem{lem}[thm]{\protect\lemmaname}
  \theoremstyle{plain}
  \newtheorem{cor}[thm]{\protect\corollaryname}

\usepackage[final,nonatbib]{nips_2017}

\usepackage{url}
\usepackage{amsfonts}
\usepackage{nicefrac}
\usepackage{microtype}

\author{
SueYeon Chung\\ Harvard University \\
 \texttt{schung@fas.harvard.edu} \\
\And
  Uri Cohen \\ Hebrew Universy of Jerusalem\\
  \texttt{uri.cohen@alice.nc.huji.ac.il} \\
\AND
  Haim Sompolinsky \\ Harvard University \\ 
  Hebrew University of Jerusalem \\
  \texttt{haim@fiz.huji.ac.il} \\
\And
  Daniel D. Lee \\
  University of Pensylvania \\ 
  Facebook AI Research \\
  \texttt{ddlee@seas.upenn.edu} \\
}

\makeatother

\usepackage{babel}
  \providecommand{\corollaryname}{Corollary}
  \providecommand{\lemmaname}{Lemma}
\providecommand{\theoremname}{Theorem}

\begin{document}

\title{Learning Data Manifolds \\
with a Cutting Plane Method}
\maketitle
\begin{abstract}
\textcolor{black}{We consider the problem of classifying data manifolds
where each manifold represents invariances that are parameterized
by continuous degrees of freedom. Conventional data augmentation methods
rely upon sampling large numbers of training examples from these manifolds;
instead, we propose an iterative algorithm called $M_{CP}$ based
upon a cutting-plane approach that efficiently solves a quadratic
semi-infinite programming problem to find the maximum margin solution.
We provide a proof of convergence as well as a polynomial bound on
the number of iterations required for a desired tolerance in the objective
function. The efficiency and performance of $M_{CP}$ are demonstrated
in high-dimensional simulations and on image manifolds generated from
the ImageNet dataset. Our results indicate that $M_{CP}$ is able
to rapidly learn good classifiers and shows superior generalization
performance compared with conventional maximum margin methods using
data augmentation methods.}
\end{abstract}
\textcolor{black}{}

\section{\textcolor{black}{Introduction }}

\textcolor{black}{Handling object variability is a major challenge
for machine learning systems. For example, in visual recognition tasks,
changes in pose, lighting, identity or background can result in large
variability in the appearance of objects \cite{hinton1997modeling}.
Techniques to deal with this variability has been the focus of much
recent work, especially with convolutional neural networks consisting
of many layers. The manifold hypothesis states that natural data variability
can be modeled as lower-dimensional manifolds embedded in higher dimensional
feature representations \cite{bengio2013representation}. A deep neural
network can then be understood as disentangling or flattening the
data manifolds so that they can be more easily read out in the final
layer \cite{brahma2016deep}. Manifold representations of stimuli
have also been utilized in neuroscience, where different brain areas
are believed to untangle and reformat their representations \cite{riesenhuber1999hierarchical,serre2005object,hung2005fast,dicarlo2007untangling,pagan2013signals}.}

\textcolor{black}{This paper addresses the problem of classifying
data manifolds that contain invariances with a number of continuous
degress of freedom. These invariances may be modeled using prior knowledge,
manifold learning algorithms \cite{tenenbaum1998mapping,roweis2000nonlinear,tenenbaum2000global,belkin2003laplacian,belkin2006manifold,canas2012learning}
or as generative neural networks via adversarial training \cite{goodfellow2014generative}.
Based upon knowledge of these structures, other work has considered
building group-theoretic invariant representations \cite{anselmi2013unsupervised}
or constructing invariant metrics \cite{simard1994memory}. On the
other hand, most approaches today rely upon data augmentation by explicitly
generating ``virtual'' examples from these manifolds \cite{niyogi1998incorporating}\cite{scholkopf1996incorporating}.
Unfortunately, the number of samples needed to successfully learn
the underlying manifolds may increase the original training set by
more than a thousand-fold \cite{krizhevsky2012imagenet}.}
\begin{figure}[h]
\noindent \begin{centering}
\textcolor{black}{\includegraphics[width=0.3\columnwidth]{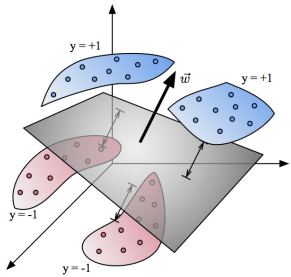}}
\par\end{centering}
\textcolor{black}{\caption{The maximum margin binary classification problem for a set of manifolds.
The optimal linear hyperplane is parameterized by the weight vector
$\vec{w}$ which separates positively labeled manifolds from negatively
labeled manifolds. Conventional data augmentation techniques resort
to sampling a large number of points from each manifold to train a
classifier. \label{fig:ManifoldClassification} }
}
\end{figure}

We propose a new method, called the Manifold Cutting Plane algorithm
or $M_{CP}$, that uses knowledge of the manifolds to efficiently
learn a maximum margin classifier. Figure \ref{fig:ManifoldClassification}
illustrates the problem in its simplest form, binary classification
of manifolds with a linear hyperplane with extensions to this basic
model discussed later. Given a number of manifolds embedded in a feature
space, the $M_{CP}$ algorithm learns a weight vector $\vec{w}$ that
separates positively labeled manifolds from negatively labeled manifolds
with the maximum margin. Although the manifolds consist of uncountable
sets of points, the $M_{CP}$ algorithm is able to find a good solution
in a provably finite number of iterations and training examples.

Support vector machines (SVM) can learn a maximum margin classifier
given a finite set of training examples \textcolor{black}{\cite{vapnik1998statistical}};
however, with conventional data augmentation methods, the number of
training examples increase exponentially rendering the standard SVM
algorithm intractable. Methods such as shrinkage and chunking to reduce
the complexity of SVM have been studied before in the context of dealing
with large-scale datasets \cite{smola1998learning}, but the resultant
kernel matrix may still be very large. Other methods which subsample
the kernel matrix\textcolor{black}{{} \cite{lee2001rsvm}} or reduce
the number of training samples \textcolor{black}{\cite{wang2005training}},\cite{smola1998learning}
may result in suboptimal solutions that do not generalize well. 

Our $M_{CP}$ algorithm directly handles the uncountable set of points
in the manifolds by solving a quadratic semi-infinite programming
problem (QSIP). $M_{CP}$ is based upon a cutting-plane method which
iteratively refines a finite set of training examples to solve the
underlying QSIP  \cite{fang2001solving,kortanek1993central,liu2004new}.
The cutting-plane method was also previously shown to efficiently
handle learning problems with a finite number of examples but an exponentially
large number of constraints \textcolor{black}{\cite{joachims2006training}}.
We provide a novel analysis of the convergence of $M_{CP}$ with both
hard and soft margins. When the problem is realizable, the convergence
bound explicitly depends upon the margin value whereas with a soft
margin and slack variables, the bound depends linearly on the number
of manifolds.

\textcolor{black}{The paper is organized as follows. We first consider
the hard margin problem and analyze the simplest form of the $M_{CP}$
algorithm. Next, we introduce slack variables in $M_{CP}$, one for
each manifold, and analyze its convergence with additional auxiliary
variables. We then demonstrate the application of $M_{CP}$ to both
high-dimensional synthetic data manifolds and to feature representations
of images undergoing a variety of warpings. We compare its performance,
both in efficiency and generalization error, with conventional SVMs
using data augmentation techniques. Finally, we discuss some natural
extensions and potential future work on $M_{CP}$ and its applications.}

\section{\textcolor{black}{Manifolds Cutting Plane Algorithm with Hard Margin
\label{sec:M4-simple}}}

\textcolor{black}{In this section, we first consider the problem of
classifying a set of manifolds when they are linearly separable. This
allows us to introduce the simplest version of the $M_{CP}$ algorithm
along with the appropriate definitions and QSIP formulation. We analyze
the convergence of the simple algorithm and prove an upper bound on
the number of errors the algorithm can make in this setting.}

\subsection{\textcolor{black}{Hard Margin QSIP \label{subsec:simpeM4-optprob}}}

\textcolor{black}{Formally, we are given a set of $P$ manifolds $M_{p}\subset\mathbb{R}^{N}$,
$p=1,\ldots,P$ with binary labels $y_{p}=\pm1$ (all points in the
same manifold share the same label). Each manifold $M_{p}$ is defined
by $\vec{x}=M_{p}(\vec{s}$) where $\vec{s}\in S_{p}$, $S_{p}$ is
a compact, convex subset of $\mathbb{R}^{D}$ representing the parameterization
of the invariances and $M_{p}(\vec{s}):\mathbb{R}^{D}\rightarrow\mathbb{R}^{N}$
is a continuous function of $\vec{s}\in S_{p}$ so that the manifolds
are bounded: $\forall\vec{s},\,\left\Vert M_{p}(\vec{s})\right\Vert <L$
by some $L$. We would like to solve the following semi-infinite quadratic
programming problem for the weight vector $\vec{w}\in\mathbb{R}^{N}$
: }

\begin{equation}
SVM^{simple}:\underset{\vec{w}}{\text{argmin}}\frac{1}{2}\left\Vert \vec{w}\right\Vert ^{2}\,s.t.\;\forall p,\forall\vec{x}\in M_{p}:\;y_{p}\left\langle \vec{w},\vec{x}\right\rangle \ge1\label{eq:SVMopt}
\end{equation}

\textcolor{black}{This is the primal formulation of the problem, where
maximizing the margin $\kappa=\frac{1}{||\vec{w}||}$ is equivalent
to minimizing the squared norm $\frac{1}{2}||\vec{w}||^{2}.$ We denote
the maximum margin attainable by $\kappa^{*}$, and the optimal solution
as $\vec{w}^{*}$. For simplicity, we do not explicitly include the
bias term here. A non-zero bias can be modeled by adding an additional
feature of constant value as a component to all the $\vec{x}$. Note
that the dual formulation of this QSIP is more complicated, involving
optimization of non-negative measures over the manifolds. In order
to solve the hard margin QSIP, we propose the following simple $M_{CP}$
algorithm. }

\subsection{\textcolor{black}{$M_{CP}^{simple}$ Algorithm}}

\textcolor{black}{The $M_{CP}^{simple}$ algorithm is an iterative
algorithm to find the optimal $\vec{w}$ in (\ref{eq:SVMopt}), based
upon a cutting plane method for solving the QSIP. The general idea
behind $M_{CP}^{simple}$ is to start with a finite number of training
examples, find the maximum margin solution for that training set,
augment the training set by finding a point on the manifolds that
violates the constraints, and iterating this process until a tolerance
criterion is reached.}

\textcolor{black}{At each stage $k$ of the algorithm there is a finite
set of training points and associated labels. The training set at
the $k$-th iteration is denoted by the set: $T_{k}=\left\{ \left(\vec{x}^{i}\in M_{p_{i}},y_{p_{i}}\right)\right\} $
with $i=1,\ldots,\left|T_{k}\right|$ examples. For the $i$-th pattern
in $T_{k}$, $p_{i}$ is the index of the manifold, and $y_{p_{i}}$
is its associated label.}

\textcolor{black}{On this set of examples, we solve the following
finite quadratic programming problem:}

\begin{equation}
SVM_{T_{k}}:\underset{\vec{w}}{\text{argmin}}\frac{1}{2}\left\Vert \vec{w}\right\Vert ^{2}\ s.t.\forall\vec{x}^{i}\in T_{k}:\;y_{p_{i}}\left\langle \vec{w},\vec{x}^{i}\right\rangle \ge1\label{eq:SVMTk}
\end{equation}

\textcolor{black}{to obtain the optimal weights $\vec{w}^{(k)}$ on
the training set $T_{k}$. We then find a constraint-violating point
$\vec{x}_{k+1}\in M_{p_{k+1}}$ from one of the manifolds such that
\begin{equation}
\left\{ p_{k+1},\,\vec{x}_{k+1}\in M_{p_{k+1}}\right\} :\quad y_{p+1}\left\langle \vec{w}^{(k)},\vec{x}_{p_{k+1}}\right\rangle <1-\text{\ensuremath{\delta}}\label{eq:violation_point-1}
\end{equation}
}

\textcolor{black}{with a required tolerance $\delta>0.$ If there
is no such point, the $M_{CP}^{simple}$ algorithm terminates. If
such a point exists, it is added to the training set, defining the
new set $T_{k+1}=T_{k}\cup\left\{ \left(\vec{x}_{k+1},y_{p_{k+1}}\right)\right\} $.
The algorithm then proceeds at the next iteration to solve $SVM_{T_{k+1}}$to
obtain $\vec{w}^{(k+1)}$. For $k=1$, the set $T_{1}$ is initialized
with at least one point from each manifold. The pseudocode for $M_{CP}^{simple}$
is shown in Alg. \ref{alg:M4}. }

\textcolor{black}{}
\begin{algorithm}[h]
\textcolor{black}{Input: $\delta$ (tolerance), $P$ manifolds and
labels $\left\{ M_{p},\,y_{p}=\pm1\right\} $, $p=1,...,P$.}

\textcolor{black}{1. Initialize $k=1$, and the set $T_{1}=\left\{ \left(\vec{x}^{i}\in M_{p_{i}},\,y_{p_{i}}\right)\right\} $
with at least one sample from each manifold $M_{p}$.}

\textcolor{black}{2. Solve for $\vec{w}^{(k)}$ in $SVM_{T_{k}}$:
$\min\frac{1}{2}\left\Vert \vec{w}\right\Vert ^{2}$ s.t. $y^{p_{i}}\langle\vec{w,}\vec{x}^{i}\rangle\ge1$
for all $\left(\vec{x}^{i},y^{p_{i}}\right)\in T_{k}$.}

\textcolor{black}{3. Find a point $\vec{x}_{k+1}\in M_{p_{k+1}}$
among the manifolds $\left\{ p_{k+1}=1,...,P\right\} $ with a margin
s.t. $y_{p_{k+1}}\left\langle \vec{w}^{(k)},\vec{x}_{k+1}\right\rangle <1-\text{\ensuremath{\delta}}$. }

\textcolor{black}{4. If there is no such point, then stop. Else, augment
the point set: $T_{k+1}=T_{k}\cup\left\{ \left(\vec{x}_{k+1},y_{p_{k+1}}\right)\right\} $.}

\textcolor{black}{5. $k\leftarrow k+1$ and go to 2. }

\textcolor{black}{\caption{Pseudocode for the $M_{CP}^{simple}$ algorithm.\label{alg:M4} }
}
\end{algorithm}

\textcolor{black}{In step 3 of the $M_{CP}^{simple}$ algorithm, a
point among the manifolds that violates the margin constraint needs
to be found. The use of a ``separation oracle'' is common in other
cutting plane algorithms such as those used for structural SVM's \cite{joachims2006training}
or linear mixed-integer programming \cite{marchand2002cutting}. In
our case, this requires determining the feasibility of $y_{p}\left\langle \vec{w}^{(k)},M_{p}(\vec{s})\right\rangle <1-\delta$
over the $D$-dimensional convex parameter set $\vec{s}\in S_{p}$.
When the manifold mapping is convex, feasibility can be determined
sometimes analytically or more generally by a variety of modern convex
optimization techniques. For non-convex mappings, a feasible separating
point can be found with search techniques in the $D$-dimensional
parameter set using gradient information or finite differences or
approximately via convex relaxation techniques. In our experiments
below, we provide some specific examples of how separating points
can be efficiently found.}

\subsection{\textcolor{black}{Convergence of $M_{CP}^{simple}$ }}

\textcolor{black}{The $M_{CP}^{simple}$ algorithm will converge asymptotically
to an optimal solution when it exists. Here we show that the $\vec{w}^{(k)}$
asymptotically converges to an optimal $\vec{w}^{\star}$. Denote
the change in the weight vector in the $k$-th iteration as $\Delta\vec{w}^{(k)}=\vec{w}^{(k+1)}-\vec{w}^{(k)}$.
We present a set of lemmas and theorems leading up to the bounds on
the number of iterations for convergence, and the estimation of the
objective function. More detailed proofs can be found in Supplementary
Materials (SM).}
\begin{lem}
\textcolor{black}{\label{lem:delta_projection_simple}The change in
the weights satisfies $\langle\Delta\vec{w}^{(k)},\vec{w}^{(k)}\rangle\geq0$
. }
\end{lem}
\begin{proof}
\textcolor{black}{Define $\vec{w}(\lambda)=\vec{w}^{(k)}+\lambda\Delta\vec{w}^{(k)}$.
Then for all $0\le\lambda\le1$, $\vec{w}(\lambda)$ satisfies the
constraints on the point set $T_{k}$: $y_{p_{i}}\left\langle \vec{w}(\lambda),\vec{x}_{i}\right\rangle \ge1$
for all $\left(\vec{x}_{i},y_{p_{i}}\right)\in T_{k}$. However, if
$\langle\Delta\vec{w}^{(k)},\vec{w}^{(k)}\rangle<0$, there exists
a $0<\lambda^{\prime}<1$ such that $\left\Vert \vec{w}(\lambda^{\prime})\right\Vert ^{2}<\left\Vert \vec{w}^{(k)}\right\Vert ^{2}$,
contradicting the fact that $\vec{w}^{(k)}$ is a solution to $SVM_{T_{k}}$.}
\end{proof}
\textcolor{black}{Next, we show that the norm $\left\Vert \vec{w}^{(k)}\right\Vert ^{2}$
must monotonically increase by a finite amount at each iteration.} 
\begin{thm}
\textcolor{black}{\label{thm:delta_finite_simple}In the $k_{th}$
iteration of $M_{CP}^{simple}$ algorithm, the increase in the norm
of $\vec{w}^{(k)}$ is lower bounded by $\left\Vert \vec{w}^{(k+1)}\right\Vert ^{2}\ge\left\Vert \vec{w}^{(k)}\right\Vert ^{2}+\frac{\delta_{k}^{2}}{L^{2}}$,
where $\delta_{k}=1-y_{p_{k+1}}\left\langle \vec{w}^{(k)},\vec{x}_{k+1}\right\rangle $
and }$\left\Vert \vec{x}_{k+1}\right\Vert \leq L$\textcolor{black}{{}
. }
\end{thm}
\begin{proof}
\textcolor{black}{First, note that $\delta_{k}>\delta\geq0$ , otherwise
the algorithm stops. We have: }$\left\Vert \vec{w}^{(k+1)}\right\Vert ^{2}$$=\left\Vert \vec{w}^{(k)}\right\Vert ^{2}+\left\Vert \Delta\vec{w}^{(k)}\right\Vert ^{2}+2\left\langle \vec{w}^{(k)},\Delta\vec{w}^{(k)}\right\rangle $$\geq\left\Vert \vec{w}^{(k)}\right\Vert ^{2}+\left\Vert \Delta\vec{w}^{(k)}\right\Vert ^{2}$\textcolor{black}{($\because$Lemma
\ref{lem:delta_projection_simple}). Consider the point added to set
$T_{k+1}=T_{k}\cup\left(\vec{x}_{k+1},y_{p_{k+1}}\right)$. At this
point, $y_{p_{k+1}}\left\langle \vec{w}^{(k+1)},\vec{x}_{k+1}\right\rangle \ge1$,
$y_{p_{k+1}}\left\langle w^{(k)},\vec{x}_{k+1}\right\rangle =1-\delta_{k}$,
hence $y_{p_{k+1}}\left\langle \Delta\vec{w}^{(k)},\vec{x}_{k+1}\right\rangle \ge\delta_{k}$.
Then, from the Cauchy-Schwartz inequality,
\begin{equation}
\left\Vert \Delta\vec{w}^{(k)}\right\Vert ^{2}\ge\frac{\delta_{k}^{2}}{\left\Vert \vec{x}_{k+1}\right\Vert ^{2}}>\frac{\delta_{k}^{2}}{L^{2}}>\frac{\delta^{2}}{L^{2}}\label{eq:deltaWLowerBound}
\end{equation}
}

\textcolor{black}{Since the solution $\vec{w}^{\star}$ satisfies
the constraints for $T_{k}$, $\left\Vert \vec{w}^{(k)}\right\Vert \le\frac{1}{\kappa^{*}}$.
Thus, the sequence of iterations monotonically increase norms and
are upper bounded by $\frac{1}{\kappa^{\star}}$. Due to convexity,
there is a single global optimum and the $M_{simple}^{4}$ algorithm
is guaranteed to converge. }
\end{proof}
\textcolor{black}{As a corollary, we see that this procedure is guaranteed
to find a realizable solution if it exists in a finite number of steps.}
\begin{cor}
\textcolor{black}{\label{cor:ZeroTrErrConv}The $M_{CP}^{simple}$
algorithm converges to a zero error classifier in less than $\frac{L^{2}}{\left(\kappa^{\star}\right)^{2}}$
iterations, where $\kappa^{\star}$ is the optimal margin and $L$
bounds the norm of the points on the manifolds.}
\end{cor}
\begin{proof}
\textcolor{black}{When there is an error, we have $\delta_{k}>1$,
and $\left\Vert \vec{w}^{(k+1)}\right\Vert ^{2}\ge\left\Vert \vec{w}^{(k)}\right\Vert ^{2}+\frac{1}{L^{2}}$
(See (\ref{eq:deltaWLowerBound})). This implies the total number
of possible errors is upper bounded by $\frac{L^{2}}{\left(\kappa^{\star}\right)^{2}}$. }
\end{proof}
\textcolor{black}{With a finite tolerance $\delta>0$, we obtain a
bound on the number of iterations required for convergence:}
\begin{cor}
\textcolor{black}{\label{terminationBound}The $M_{CP}^{simple}$
algorithm for a given tolerance $\delta>0$ terminates in less than
$\frac{L^{2}}{\left(\kappa^{\star}\delta\right)^{2}}$ iterations
where $\kappa^{\star}$ is the optimal margin and $L$ bounds the
norm of the points on the manifolds.}
\end{cor}
\begin{proof}
\textcolor{black}{Again, $\left\Vert \vec{w}^{k}\right\Vert ^{2}\le\left\Vert \vec{w}^{\star}\right\Vert ^{2}=\frac{1}{\left(\kappa^{\star}\right)^{2}}$
and each iteration increases the squared norm by at least $\frac{\delta^{2}}{L^{2}}$.}
\end{proof}
\textcolor{black}{We can also bracket the error in the objective function
after $M_{CP}^{simple}$ terminates:}
\begin{cor}
\textcolor{black}{\label{w_estimation}With tolerance $\delta>0$,
after $M_{CP}^{simple}$ terminates with solution $\vec{w}_{M_{CP}}$,
the optimal value $\left\Vert \vec{w}^{\star}\right\Vert $ of $SVM^{simple}$
is bracketed by: $\left\Vert \vec{w}_{M_{CP}}\right\Vert ^{2}\le\left\Vert \vec{w}^{\star}\right\Vert ^{2}\le\frac{1}{\left(1-\delta\right)^{2}}\left\Vert \vec{w}_{M_{CP}}\right\Vert ^{2}$. }
\end{cor}
\begin{proof}
\textcolor{black}{The lower bound on $\left\Vert \vec{w}^{\star}\right\Vert ^{2}$
is as before. Since $M_{CP}^{simple}$ has terminated, setting $\vec{w}^{\prime}=\frac{1}{(1-\delta)}\vec{w}_{M_{CP}}$
would make $\vec{w}^{\prime}$ feasible for $SVM^{simple}$, resulting
in the upper bound on $\left\Vert \vec{w}^{\star}\right\Vert ^{2}$. }
\end{proof}

\section{\textcolor{black}{$M_{CP}$ with Slack Variables \label{sec:M4-slack}}}

\textcolor{black}{In many classification problems, the manifolds may
not be linearly separable due to their dimensionality, size, or correlation
structure. In these situations, $M_{CP}^{simple}$ will not be able
to find a feasible solution. To handle these problems, the classic
approach is to introduce slack variables on each point ($SVM^{slack}$).
Unfortunately, this approach requires integrations over entire manifolds
with an appropriate measure defined by the infinite set of slack variables.
Thus, we formulate an alternative version of the QSIP with slack variables
below.}

\subsection{\textcolor{black}{QSIP with Manifold Slacks }}

In this work, we propose using only one slack variable per manifold
for classification problems with non-separable manifolds. This formulation
demands that all the points on each manifold $\vec{x}\in M_{p}$ obey
an inequality constraint with one manifold slack variable, $y_{p}\left\langle \vec{w},\vec{x}\right\rangle +\xi_{p}\ge1$.
As we see below, solving for this constraint is tractable and the
algorithm has good convergence guarantees.

However, a single slack requirement for each manifold by itself may
not be sufficient for good generalization performance. Our empirical
studies show that generalization performance can be improved if we
additionally demand that some representative points $\vec{x}_{p}\in M_{p}$
on each manifold also obey the margin constraint: $y_{p}\left\langle \vec{w},\vec{x}_{p}\right\rangle \ge1$.
In this work, we implement this intuition by specifying appropriate
center points $\vec{x}_{p}^{c}$ for each manifold $M_{p}$. This
center point could be the center of mass of the manifold, a representative
point, or an exemplar used to generate the manifolds \textcolor{black}{\cite{krizhevsky2012imagenet}}.
Additional slack variables for these constraints could potentially
be introduced; in the present work, we simply demand that these points
strictly obey the margin inequalities corresponding to their manifold
label. Formally, the QSIP optimization problem is summarized below,
where the objective function is minimized over the\textcolor{black}{{}
weight vector $\vec{w}\in\mathbb{R}^{N}$ and slack variables $\vec{\xi}\in\mathbb{R}^{P}$: }

\[
\begin{array}{c}
SVM_{manifold}^{slack}:\,\underset{\vec{w},\vec{\xi}}{\text{argmin}}F(\vec{w},\vec{\xi})=\frac{1}{2}\left\Vert \vec{w}\right\Vert ^{2}+C\sum_{p=1}^{P}\xi_{p}\\
s.t.\;\forall p,\forall\vec{x}\in M_{p}:\;y_{p}\left\langle \vec{w},\vec{x}\right\rangle +\xi_{p}\ge1\text{(manifolds)},\forall p:\ y_{p}\left\langle \vec{w},\vec{x}_{p}^{c}\right\rangle \ge1\ \text{(centers)},\xi_{p}\ge0
\end{array}
\]

\subsection{\textcolor{black}{$M_{CP}^{slack}$ Algorithm }}

With these definitions, we introduce our $M_{slack}^{4}$ algorithm
with slack variables below. 

\textcolor{black}{}
\begin{algorithm}[h]
\textcolor{black}{Input: $\delta$ (tolerance), $P$ manifolds and
labels $\left\{ M_{p},y_{p}=\pm1\right\} $, and centers $\vec{x}_{p}^{c}$}

\textcolor{black}{1. Initialize $k=1$, and the set $T_{1}=\left\{ \left(\vec{x}^{i}\in M_{p_{i}},\,y_{p_{i}}\right)\right\} $
with at least one sample from each manifold $M_{p}$.}

\textcolor{black}{2. Solve for $\vec{w}^{(k)}$,$\vec{\xi}^{(k)}$:
$\min\frac{1}{2}\left\Vert \vec{w}\right\Vert ^{2}+C\sum_{p=1}^{P}\xi_{p}$
s.t. $y_{p_{\mu}}\langle\vec{w,}\vec{x}^{\mu}\rangle+\xi_{p_{\mu}}\ge1$
for all $\left(\vec{x}^{\mu},y_{p_{\mu}}\right)\in T_{k}$ and $y_{p}\langle\vec{w,}\vec{x}_{p}^{c}\rangle\ge1$
for all$p$.}

\textcolor{black}{3. Find a point $\vec{x}^{k+1}\in M_{p_{k+1}}$
among the manifolds $\left\{ p=1,...,P\right\} $ with slack violation
larger than the tolerance $\delta$: $y_{p_{k+1}}\left\langle \vec{w}^{(k)},\vec{x}_{k+1}\right\rangle +\xi_{p_{k+1}}<1-\text{\ensuremath{\delta}}$}

\textcolor{black}{4. If there is no such point, then stop. Else, augment
the point set: $T_{k+1}=T_{k}\cup\left\{ \left(\vec{x}_{k+1},y_{p_{k+1}}\right)\right\} $.}

\textcolor{black}{5. $k\leftarrow k+1$ and go to 2. }

\textcolor{black}{\caption{Pseudocode for the $M_{CP}^{slack}$ algorithm.\label{alg:M4-slack-L1} }
}
\end{algorithm}

\textcolor{black}{The proposed $M_{CP}^{slack}$ algorithm modifies
the cutting plane approach to solve a semi-infinite, semi-definite
quadratic program. Each iteration involves a finite set: $T_{k}=\left\{ \left(\vec{x}^{i}\in M_{p_{i}},y_{p_{i}}\right)\right\} $
with $i=1,\ldots,\left|T_{k}\right|$ examples that is used to define
the following soft margin SVM:}

\[
\begin{array}{c}
SVM_{T_{k}}^{slack}:\underset{\vec{w},\vec{\xi}}{\text{argmin}}\frac{1}{2}\left\Vert \vec{w}\right\Vert ^{2}+C\sum_{p=1}^{P}\xi_{p}\\
s.t.\;\forall\left(\vec{x}^{i},y_{p_{i}}\right)\in T_{k}:\;y_{p_{i}}\left\langle \vec{w},\vec{x}^{i}\right\rangle +\xi_{p_{i}}\ge1;\forall p:\ y_{p}\left\langle \vec{w},\vec{x}_{p}^{c}\right\rangle \ge1\;(centers),\xi_{p}\ge0
\end{array}
\]

\textcolor{black}{to obtain the weights $\vec{w}^{(k)}$ and slacks
$\vec{\xi}^{(k)}$ at each iteration. We then find a point $\vec{x}_{k+1}\in M_{p_{k+1}}$
from one of the manifolds so that:
\begin{equation}
y_{p_{k+1}}\left\langle \vec{w}^{(k)},\vec{x}_{k+1}\right\rangle +\xi_{p_{k+1}}^{(k)}=1-\delta_{k}\label{eq:violation_point}
\end{equation}
where $\delta_{k}>\delta$. If there is no such a point, the $M_{CP}^{slack}$
algorithm terminates. Otherwise, the point $\vec{x}_{k+1}$ is added
as a training example to the set $T_{k+1}=T_{k}\cup\left\{ \left(\vec{x}_{k+1},y_{p_{k+1}}\right)\right\} $
and the algorithm proceeds to solve $SVM_{T_{k+1}}^{slack}$ to obtain
$\vec{w}^{(k+1)}$ and $\vec{\xi}^{(k+1)}$.}

\subsection{\textcolor{black}{Convergence of $M_{CP}^{slack}$ }}

\textcolor{black}{Here we show that the objective function $F\left(\vec{w},\vec{\xi}\right)=\frac{1}{2}\left\Vert \vec{w}\right\Vert ^{2}+C\sum_{p=1}^{P}\xi_{p}$
is guaranteed to increase by a finite amount with each iteration.
This result is similar to \cite{tsochantaridis2005large}, but here
we present statements in the primal domain over an infinite number
of examples. More detailed proofs can be found in SM.}
\begin{lem}
\textcolor{black}{\label{lem:delta_projection_slack} The change in
the weights and slacks satisfy:
\begin{equation}
\left\langle \Delta\vec{w}^{(k)},\vec{w}^{(k)}\right\rangle +C\sum_{p}\Delta\vec{\xi}_{p}^{(k)}\ge0\label{eq:delta_projection}
\end{equation}
}

where\textcolor{black}{{} $\Delta\vec{w}^{(k)}=\vec{w}^{(k+1)}-\vec{w}^{(k)}$
and $\Delta\vec{\xi}^{(k)}=\vec{\xi}^{(k+1)}-\vec{\xi}^{(k)}$.}
\end{lem}
\begin{proof}
\textcolor{black}{Define $\vec{w}(\lambda)=\vec{w}^{(k)}+\lambda\Delta\vec{w}^{(k)}$
and $\vec{\xi}(\lambda)=\vec{\xi}^{(k)}+\lambda\Delta\vec{\xi}^{(k)}$.
Then for all $0\le\lambda\le1$, $\vec{w}(\lambda)$ and $\vec{\xi}(\lambda)$
satisfy the constraints for $SVM_{T_{k}}^{slack}$. The resulting
change in the objective function is given by:
\begin{multline}
F\left(\vec{w}(\lambda),\vec{\xi}(\lambda)\right)-F\left(\vec{w}^{(k)},\vec{\xi}^{(k)}\right)=\\
\lambda\left[\left\langle \Delta\vec{w}^{(k)},\vec{w}^{(k)}\right\rangle +C\sum_{p}\Delta\xi_{p}^{(k)}\right]+\frac{1}{2}\lambda^{2}\left\Vert \Delta\vec{w}^{(k)}\right\Vert ^{2}
\end{multline}
If (\ref{eq:delta_projection}) is not satisfied, then there is some
$0<\lambda^{\prime}<1$ such that $F\left(\vec{w}(\lambda^{\prime}),\vec{\xi}(\lambda^{\prime})\right)<F\left(\vec{w}^{(k)},\vec{\xi}^{(k)}\right)$,
which contradicts the fact that $\vec{w}^{(k)}$ and $\vec{\xi}^{(k)}$
are a solution to $SVM_{T_{k}}$.}
\end{proof}
We derive that the added point at each iteration must be a support
vector for the next weight:
\begin{lem}
\begin{singlespace}
\textcolor{black}{\label{lem:new_support_vector} In each iteration
of $M_{CP}^{slack}$ algorithm, the added point $\left(\vec{x}_{k+1},y_{p_{k+1}}\right)$
must be a support vector for the new weights and slacks, s.t. 
\begin{equation}
y_{p_{k+1}}\left\langle \vec{w}^{(k+1)},\vec{x}_{k+1}\right\rangle +\xi_{p_{k+1}}^{(k+1)}=1\label{eq:new_support}
\end{equation}
}
\end{singlespace}

\textcolor{black}{
\begin{equation}
y_{p_{k+1}}\left\langle \Delta\vec{w}^{(k)},\vec{x}_{k+1}\right\rangle +\Delta\xi_{p_{k+1}}^{(k)}=\delta_{k}\label{eq:new_support_delta}
\end{equation}
}
\end{lem}
\begin{proof}
\textcolor{black}{Suppose $y_{p_{k+1}}\left\langle \vec{w}^{(k+1)},\vec{x}_{k+1}\right\rangle +\xi_{p_{k+1}}^{(k+1)}=1+\epsilon$
for some $\epsilon>0$. Then we can choose $\lambda^{\prime}=\frac{\delta_{k}}{\delta_{k}+\epsilon}<1$
so that $\vec{w}(\lambda^{\prime})=\vec{w}^{(k)}+\lambda^{\prime}\Delta\vec{w}^{(k)}$
and $\vec{\xi}(\lambda^{\prime})=\vec{\xi}^{(k)}+\lambda^{\prime}\Delta\vec{\xi}^{(k)}$
satisfy the constraints for $SVM_{T_{k+1}}^{slack}$. But, from Lemma
\ref{lem:delta_projection_slack}, we have $F\left(\vec{w}(\lambda^{\prime}),\vec{\xi}(\lambda^{\prime})\right)<F\left(\vec{w}^{(k+1)},\vec{\xi}^{(k+1)}\right)$
which contradicts the fact that $\vec{w}^{(k+1)}$ and $\text{\ensuremath{\vec{\xi}}}^{(k+1)}$
are a solution to $SVM_{T_{k+1}}$. Thus, $\epsilon=0$ and the point
$\left(\vec{x}_{k+1},y_{p_{k+1}}\right)$ must be a support vector
for $SVM_{T_{k+1}}.$ (\ref{eq:new_support_delta}) results from subtracting
(\ref{eq:violation_point}) from (\ref{eq:new_support}).}
\end{proof}
\textcolor{black}{We also derive a bound on the following quadratic
function over nonnegative values:}
\begin{lem}
\textcolor{black}{\label{lem:quadratic_bound}Given $K>0$,$\delta>0$,
then $\forall x\ge0$}

\textcolor{black}{
\begin{equation}
\frac{1}{2}(x-\delta)^{2}+Kx\ge\min\left(\frac{1}{2}\delta^{2},\frac{1}{2}K\delta\right)
\end{equation}
}
\end{lem}
\begin{proof}
\textcolor{black}{The minimum value occurs when $x^{\star}=\left[\delta-K\right]_{+}$.
When $K\ge\delta$, then $x^{\star}=0$ and the minimum is $\frac{1}{2}\delta^{2}$.
When $K<\delta$, the minimum occurs at $K\left(\delta-\frac{1}{2}K\right)\ge\frac{1}{2}K\delta$.
Thus, the lower bound is the smaller of these two values.}
\end{proof}
Using the lemmas above, the lower bound on the change in the objective
function can be found: 
\begin{thm}
\textcolor{black}{\label{prop:objective_increase} In each iteration
$k$ of $M_{CP}^{slack}$ algorithm, the increase in the objective
function for $SVM_{manifold}^{slack}$ , defined as }$\Delta F^{(k)}=F\left(\vec{w}^{(k+1)},\vec{\xi}^{(k+1)}\right)-F\left(\vec{w}^{(k)},\vec{\xi}^{(k)}\right)$\textcolor{black}{,
is lower bounded by
\begin{equation}
\Delta F^{(k)}\ge\min\left(\frac{1}{8}\frac{\delta_{k}^{2}}{L^{2}},\,\frac{1}{2}C\delta_{k}\right)\label{eq:dFbound}
\end{equation}
}
\end{thm}
\begin{proof}
\textcolor{black}{We first note that the change in objective function
is strictly increasing: }

\textcolor{black}{{} 
\begin{equation}
F\left(\vec{w}^{(k+1)},\vec{\xi}^{(k+1)}\right)-F\left(\vec{w}^{(k)},\vec{\xi}^{(k)}\right)=\left[\left\langle \Delta\vec{w}^{(k)},\vec{w}^{(k)}\right\rangle +C\sum_{p}\Delta\xi_{p}^{(k)}\right]+\frac{1}{2}\left\Vert \Delta\vec{w}^{(k)}\right\Vert ^{2}>0
\end{equation}
This can be seen immediately from Lemma \ref{lem:delta_projection_slack}
when $\Delta\vec{w}^{(k)}\ne0$. On the other hand, if $\Delta\vec{w}^{(k)}=0$,
we know that $\Delta\xi_{p_{k}}^{(k)}=\delta_{k}$ from Lemma \ref{lem:new_support_vector}
and $\Delta\xi_{p\ne p_{k}}^{(k)}=0$ since $\vec{\xi}^{(k)}$ is
the solution for $SVM_{T_{k}}$. So for $\Delta\vec{w}^{(k)}=0$,
$F\left(\vec{w}^{(k+1)},\vec{\xi}^{(k+1)}\right)-F\left(\vec{w}^{(k)},\vec{\xi}^{(k)}\right)=C\delta_{k}$.
To compute a general lower bound on the increase in the objective
function, we proceed as follows.}

\textcolor{black}{The added point $\vec{x}_{k}$ comes from a particular
manifold $M_{p_{k}}$. If $\Delta\xi_{p_{k}}^{(k)}\le0$, from Lemma
\ref{lem:new_support_vector} we have $y_{p_{k}}\left\langle \Delta\vec{w}^{(k)},\vec{x}_{k}\right\rangle \ge\delta_{k}$.
Then by the Cauchy-Schwarz inequality, $\left\Vert \Delta\vec{w}^{(k)}\right\Vert ^{2}\ge\frac{\delta_{k}^{2}}{L^{2}}$
which yields $F\left(\vec{w}^{(k+1)},\vec{\xi}^{(k+1)}\right)-F\left(\vec{w}^{(k)},\vec{\xi}^{(k)}\right)\ge\frac{1}{2}\frac{\delta_{k}^{2}}{L^{2}}$.}

\textcolor{black}{We next analyze $\Delta\xi_{p_{k}}^{(k)}>0$ and
consider the finite set of points $\left(\vec{x}^{\nu},\,p_{k}\right)\in T_{k}$
that come from the $p_{k}$ manifold. There must be at least one such
point in $T_{k}$ by initialization of the algorithm. Each of these
points obeys the constraints: 
\begin{align}
y_{p_{k}}\left\langle \vec{w}^{(k)},\vec{x}^{\nu}\right\rangle +\xi_{p_{k}}^{(k)} & =1+\epsilon_{\nu}^{(k)}\\
y_{p_{k}}\left\langle \vec{w}^{(k+1)},\vec{x}^{\nu}\right\rangle +\xi_{p_{k}}^{(k+1)} & =1+\epsilon_{\nu}^{(k+1)}\\
\epsilon_{\nu}^{(k)},\,\epsilon_{\nu}^{(k+1)} & \ge0
\end{align}
}

\textcolor{black}{We consider the minimum value of the thresholds:
$\eta=\min_{\nu}\epsilon_{\nu}^{(k)}$. We have two possibilities:
$\eta$ is positive so that none of the points are support vectors
for $SVM_{T_{k}}^{slack}$, or $\eta=0$ so that at least one support
vector lies in $M_{p_{k}}$.}

\textbf{\textcolor{black}{Case $\eta>0$:}}

\textcolor{black}{In this case, we define a linear set of slack variables:
\begin{eqnarray}
\xi_{p}(\lambda) & = & \begin{cases}
\xi_{p}^{(k)}+\lambda\Delta\xi_{p}^{(k)} & p\ne p_{k}\\
\xi_{p_{k}}^{(k)} & p=p_{k}
\end{cases}
\end{eqnarray}
and weights $\vec{w}(\lambda)=\vec{w}^{(k)}+\lambda\Delta\vec{w}^{(k)}$.
Then for $0\le\lambda\le\min\left(\frac{\eta_{p_{k}}}{\Delta\xi_{p_{k}}^{(k)}},1\right)$,
$\vec{w}(\lambda)$ and $\vec{\xi}(\lambda)$ will satisfy the constraints
for $SVM_{T_{k}}$. Following similar reasoning in Lemma \ref{lem:delta_projection_slack},
this implies 
\begin{equation}
\left\langle \Delta\vec{w}^{(k)},\vec{w}^{(k)}\right\rangle +C\sum_{p\ne p_{k}}\Delta\vec{\xi}_{p}^{(k)}\ge0\label{eq:nonsv_gradient_bound-1}
\end{equation}
Then in this case, we have}

\textcolor{black}{{} 
\begin{align}
 & F\left(\vec{w}^{(k+1)},\vec{\xi}^{(k+1)}\right)-F\left(\vec{w}^{(k)},\vec{\xi}^{(k)}\right)\nonumber \\
 & =\left[\left\langle \Delta\vec{w}^{(k)},\vec{w}^{(k)}\right\rangle +C\sum_{p}\Delta\xi_{p}^{(k)}\right]+\frac{1}{2}\left\Vert \Delta\vec{w}^{(k)}\right\Vert ^{2}\nonumber \\
 & \ge C\Delta\xi_{p_{k}}^{(k)}+\frac{1}{2}\left\Vert \Delta\vec{w}^{(k)}\right\Vert ^{2}\label{eq:nonsv_gradient-1}\\
 & \ge\frac{1}{2}\frac{\left(\delta_{k}-\Delta\xi_{p_{k}}^{(k)}\right)^{2}}{L^{2}}+C\Delta\xi_{p_{k}}^{(k)}\label{eq:nonsv_cauchy_schwarz-1}\\
 & \ge\min\left(\frac{1}{2L^{2}}\delta_{k}^{2},\,\frac{1}{2}C\delta_{k}\right)\label{eq:nonsv_quadratic_bound-1}
\end{align}
by applying (\ref{eq:nonsv_gradient_bound-1}) in (\ref{eq:nonsv_gradient-1}),
Lemma \ref{lem:new_support_vector} and Cauchy-Schwarz in (\ref{eq:nonsv_cauchy_schwarz-1})
and Lemma \ref{lem:quadratic_bound} in (\ref{eq:nonsv_quadratic_bound-1})
.}

\textbf{\textcolor{black}{Case $\eta=0$:}}

\textcolor{black}{In this case, we consider $\varepsilon=\min_{\epsilon_{\nu}^{(k)}=0}\epsilon_{\nu}^{(k+1)}\ge0$,
i.e. the minimal increase among the support vectors. We then define
\begin{eqnarray}
\xi_{p}(\lambda) & = & \begin{cases}
\xi_{p}^{(k)}+\lambda\Delta\xi_{p}^{(k)} & p\ne p_{k}\\
\xi_{p_{k}}^{(k)}+\lambda\left(\Delta\xi_{p}^{(k)}-\varepsilon\right) & p=p_{k}
\end{cases}
\end{eqnarray}
and weights $\vec{w}(\lambda)=\vec{w}^{(k)}+\lambda\Delta\vec{w}^{(k)}$.
There will then be a finite range of $0\le\lambda\le\lambda_{min}$
for which $\vec{\xi}(\lambda)$ and $\vec{w}(\lambda)$ satisfy the
constraints for $SVM_{T_{k}}$ so that 
\begin{align}
\left\langle \Delta\vec{w}^{(k)},\vec{w}^{(k)}\right\rangle +C\sum_{p\ne p_{k}}\Delta\vec{\xi}_{p}^{(k)}+C\left(\Delta\xi_{p_{k}}^{(k)}-\varepsilon\right) & \ge0\\
\left\langle \Delta\vec{w}^{(k)},\vec{w}^{(k)}\right\rangle +C\sum_{p}\Delta\vec{\xi}_{p}^{(k)} & \ge C\varepsilon\label{eq:sv_gradient_bound-1}
\end{align}
}

\textcolor{black}{We also have a support vector $\left(\vec{x}^{\nu},\,p_{k}\right)\in T_{k}$
so that 
\begin{align}
y_{p_{k}}\left\langle \vec{w}^{(k+1)},\vec{x}^{\nu}\right\rangle +\xi_{p_{k}}^{(k+1)} & =1+\varepsilon\\
y_{p_{k}}\left\langle \Delta\vec{w}^{(k)},\vec{x}^{\nu}\right\rangle +\Delta\xi_{p_{k}}^{(k)} & =\varepsilon
\end{align}
Using Lemma \ref{lem:new_support_vector} and Cauchy-Schwarz, we get:
\begin{align}
y_{p_{k}}\left\langle \Delta\vec{w}^{(k)},\vec{x}_{k}-\vec{x}^{\nu}\right\rangle  & =\delta_{k}-\varepsilon\\
\left\Vert \Delta\vec{w}^{(k)}\right\Vert ^{2} & \ge\frac{1}{4L^{2}}\left(\delta_{k}-\varepsilon\right)^{2}\label{eq:sv_cauchy_schwarz-1}
\end{align}
Thus, we have: 
\begin{align}
 & F\left(\vec{w}^{(k+1)},\vec{\xi}^{(k+1)}\right)-F\left(\vec{w}^{(k)},\vec{\xi}^{(k)}\right)\\
 & =\left[\left\langle \Delta\vec{w}^{(k)},\vec{w}^{(k)}\right\rangle +C\sum_{p}\Delta\xi_{p}^{(k)}\right]+\frac{1}{2}\left\Vert \Delta\vec{w}^{(k)}\right\Vert ^{2}\\
 & \ge C\varepsilon+\frac{1}{8L^{2}}\left(\delta_{k}-\varepsilon\right)^{2}\\
 & \ge\min\left(\frac{1}{8L^{2}}\delta_{k}^{2},\,\frac{1}{2}C\delta_{k}\right)\label{eq:sv_quadratic_bound-1}
\end{align}
by applying (\ref{eq:sv_gradient_bound-1}) and (\ref{eq:sv_cauchy_schwarz-1})
to obtain the final bound.} 
\end{proof}
\textcolor{black}{Since the $M_{CP}^{slack}$ algorithm is guaranteed
to increase the objective by a finite amount, it will terminate in
a finite number of iterations if we require $\delta_{k}>\delta$ for
some positive $\delta>0$.}
\begin{cor}
\textcolor{black}{\label{slack_bound}The $M_{CP}^{slack}$ algorithm
for a given $\delta>0$ will terminate after at most $P\cdot\max\left(\frac{8CL^{2}}{\delta^{2}},\,\frac{2}{\delta}\right)$
iterations where $P$ is the number of manifolds, L bounds the norm
of the points on the manifolds.}
\end{cor}
\begin{proof}
\textcolor{black}{$\vec{w}=0$ and $\xi_{p}=1$ is a feasible solution
for $SVM_{manifold}^{slack}$. Therefore, the optimal objective function
is upper-bounded by $F\left(\vec{w}=0,\vec{\xi}=1\right)=PC$. The
upper bound on the number of iterations is then provided by Theorem
(\ref{prop:objective_increase}). }
\end{proof}
\textcolor{black}{We can also bound the error in the objective function
after $M_{CP}^{slack}$ terminates:}
\begin{cor}
\textcolor{black}{\label{slack_estimation}With $\delta>0$, after
$M_{CP}^{slack}$ terminates with solution $\vec{w}_{M_{CP}}$, slack
$\vec{\xi}_{M_{CP}}$ and value $F_{M_{CP}}=F\left(\vec{w}_{M_{CP}},\vec{\xi}_{M_{CP}}\right)$,
then the optimal value $F^{\star}$ of $SVM_{manifold}^{slack}$ is
bracketed by:
\begin{equation}
F_{M_{CP}}\le F^{\star}\le F_{M_{CP}}+PC\delta.
\end{equation}
}
\end{cor}
\begin{proof}
\textcolor{black}{The lower bound on $F^{\star}$ is apparent since
$SVM_{manifold}^{slack}$ includes $SVM_{T_{k}}^{slack}$ constraints
for all $k$. Setting the slacks $\xi_{p}=\xi_{M_{CP},p}+\delta$
will make the solution feasible for $SVM_{manifold}^{slack}$ resulting
in the upper bound. }
\end{proof}

\section{\textcolor{black}{Experiments \label{sec:Experiments}}}

\subsection{$L_{q}$ Ellipsoidal\textcolor{black}{{} Manifolds }}

\textcolor{black}{As an illustration of our method, we have generated
$D$-dimensional $L_{q}$-norm ellipsoids with random radii, centers,
and directions. The points on each manifold $M_{p}$ are parameterized
as $M_{p}=\left\{ \vec{x}\left|\vec{x}=\vec{x}_{p}^{c}+\sum_{i=1}^{D}R_{i}s_{i}\vec{u}_{i}^{p}\in\mathbb{R}^{N}\right.\right\} $
where the center $\vec{x}_{p}^{c}$ and basis vectors $\vec{u}_{i}^{p}$
are random Gaussian and $R_{i}$, the ellipsoidal radii, are sampled
from $\text{Unif}[0.5R_{0},1.5R_{0}]$ with mean $R_{0}$. }

We compare the performance of $M_{CP}$ to the conventional Point
SVM ($SVM^{simple}$, $SVM^{slack}$) with samples drawn from the
surface of the $L_{q}$ ellipsoids\textcolor{black}{{} for training
and test examples.} Performance is measured by generalization error
on the task of classifying points from positively labeled manifolds
from negatively labeled ones, as a function of the number of training
samples used during learning. 

\textcolor{black}{For these manifolds, the separation oracle of $M_{CP}$
returns points that minimize $y^{p}\left[\vec{w}\cdot\left(\vec{x}_{0}^{p}+\sum_{i=1}^{D}R_{i}s_{i}\vec{u}_{i}^{p}\right)\right]$
over the set $\left\Vert \vec{s}\right\Vert _{q}\le1$. For norms
with $q\geq1$, the solution can be expressed as $s_{i}=-\frac{\left(h_{i}^{p}\right)^{1/(q-1)}}{\left\{ \sum\left(h_{i}^{p}\right)^{q/(q-1)}\right\} ^{1/q}}$
where $h_{i}^{p}=y^{p}\vec{w}\cdot\vec{u_{i}}^{p}$. For $M_{CP}^{slack}$,
we used an additional single margin constraint per manifold given
by the center of the ellipsoids $\vec{x}_{p}^{c}$. For the parameters
shown in Fig. }\ref{fig:L2-balls-GenErr}\textcolor{black}{, we use
$N=500$ to test $M_{CP}^{simple}$ and $N=490$ for the $M_{CP}^{slack}$
simulations which are below and above the critical capacity. }The
results illustrate that $M_{CP}$ achieves low generalization error
very efficiently compared to a conventional maximum margin classifier
using sampled training examples. 

\textcolor{black}{}
\begin{figure}[h]
\noindent \begin{centering}
\textcolor{black}{\includegraphics[width=0.55\columnwidth]{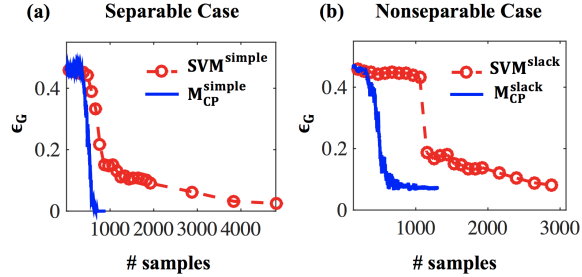}}
\par\end{centering}
\textcolor{black}{\caption{\textbf{Generalization error ($\epsilon_{G}$) of the $M_{CP}$ solution
for $L_{50}$ ellipsoids}, shown as a function of the total number
of training samples (blue solid) compared with conventional Point
SVM (red dashed). $P=48$, $D=10$, $R_{0}=20$ (mean elliptic radii)
and (a) $P/N=0.096$ just below the critical capacity is used for
$M_{CP}^{simple}$and (b) $P/N=0.098$ for $M_{CP}^{slack}$ with
slack coefficient $C_{opt}=100$. $\epsilon_{G}$ is computed from
$500$ points per manifold. \label{fig:L2-balls-GenErr}}
}
\end{figure}

\subsection{\textcolor{black}{ImageNet Object Manifolds}}

We also apply the $M_{CP}$ algorithm to a more realistic class of
object manifolds. Here each object manifold is constructed from a
set of $6-D$ affine warpings of images from the ImageNet data-set
\cite{deng2009imagenet}. The resulting $P$ manifolds are split into
dichotomies, with $\frac{1}{2}P$ manifolds considered as positive
instances and the remaining $\frac{1}{2}P$ manifolds as negative
instances. $M_{CP}$ is then used to learn a binary classifier for
the manifolds. We note that $M_{CP}$ can easily be extended to multi-class
problems but was not used in this experiment.

We compared the performace of $M_{CP}$ to conventional SVM on sampled
points from two different feature representations of the images: in
the original pixel space and in a V1-like representation of the same
images created by applying full-wave rectification after filtering
by arrays of Gabor functions \cite{serre2007robust}. We trained on
$P=4$ object manifolds, and obtained the parameter $C$ through cross-validation.
At each iteration of $M_{CP}$, a constraint-violating point on the
manifolds is found using local search among $K=5$ neighboring samples
in the affine warping space. Fig. \ref{fig:image-based-object-manifolds}
shows how quickly $M_{CP}$ converges to low generalization error
on both the Gabor features and pixel representations. We note that
the maximal number $P$ of manifolds that are separable depends upon
the feature representation. Thus, $M_{CP}$ can also be used to investigate
the benefits of alternative features representations, such as those
found in different areas of brain visual processing or in the layers
of a deep neural network.

\textcolor{black}{}
\begin{figure}[h]
\noindent \begin{centering}
\textcolor{black}{\includegraphics[width=0.75\columnwidth]{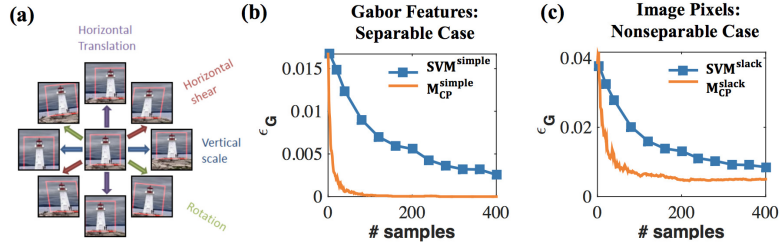}}
\par\end{centering}
\textcolor{black}{\caption{\label{fig:image-based-object-manifolds}\textbf{Image-based object
manifolds}. (a) Basic affine transformation: a template image (middle)
with an object bounding box (pink) surrounded with changes along 4
axes defined by basic affine transformation. (b-c) Generalization
error of the $M_{CP}$ solution for 6-D image-based object manifolds
as a function of the number of training samples (solid orange line)
compared with that of conventional Point SVM (blue squares), obtained
in (b) Gabor representations and (c) Pixel representations. The generalization
error is averaged over all possible choices for labels.}
}
\end{figure}

\section{\textcolor{black}{Discussion }}

\textcolor{black}{We described and analyzed a novel algorithm, $M_{CP}$,
based upon the cutting-plane method for finding the maximum margin
solution for classifying manifolds. We proved the convergence of $M_{CP}$,
and provided bounds on the number of iterations required. Our experiments
with both synthetic data manifolds and image manifolds demonstrate
the efficiency of $M_{CP}$ and superior performance in terms of generalization
error, compared to conventional SVM's using many virtual examples.}

\textcolor{black}{There is a natural extension of $M_{CP}$ to nonlinear
classifiers via the kernel trick, as all our operations involve dot
products between the weight vector $\vec{w}$ and manifold points
$M_{p}(\vec{s})$. At each iteration, the algorithm would solve the
dual version of the $SVM_{T_{k}}$ problem which is readily kernelized.
More theoretical work is needed with infinite-dimensional kernels
when the manifold optimization problem no longer is a QSIP but becomes
a fully (doubly) infinite quadratic programming problem.}

\textcolor{black}{Beyond binary classification, variations of $M_{CP}$
can also be used to solve other machine learning problems including
multi-class classification, ranking, ordinal regression, one-class
learning, etc. We can also use $M_{CP}$ to evaluate the computational
benefits of manifold representations at successive layers of deep
networks in both machine learning and in brain sensory hierarchies.
We anticipate using $M_{CP}$ to help construct novel hierarchical
architectures that can incrementally reformat the manifold representations
through the layers for better overall performance in machine learning
tasks, improving our understanding on how neural architectures can
learn to process high dimensional real-world signal ensembles and
cope with large variability due to continuous modulation of the underlying
physical parameters.}

\subsection*{Acknowledgments}

The work is partially supported by the Gatsby Charitable Foundation,
the Swartz Foundation, the Simons Foundation (SCGB Grant No. 325207),
the NIH, and the Human Frontier Science Program (Project RGP0015/2013).
D. Lee also acknowledges the support of the US National Science Foundation,
Army Research Laboratory, Office of Naval Research, Air Force Office
of Scientific Research, and Department of Transportation.

\textcolor{black}{\bibliographystyle{unsrt}
\bibliography{mmmm}
}
\end{document}